\documentclass{article}

\usepackage{microtype}
\usepackage{graphicx}
\usepackage{subcaption}
\usepackage{booktabs}
\usepackage{hyperref}

\usepackage[preprint]{icml2026}

\usepackage{amsmath}
\usepackage{amssymb}
\usepackage{mathtools}
\usepackage{amsthm}

\usepackage[capitalize,noabbrev]{cleveref}

\usepackage{algorithm}
\usepackage{algorithmic}

\theoremstyle{plain}
\newtheorem{theorem}{Theorem}[section]
\newtheorem{proposition}[theorem]{Proposition}
\newtheorem{lemma}[theorem]{Lemma}
\newtheorem{corollary}[theorem]{Corollary}
\theoremstyle{definition}
\newtheorem{definition}[theorem]{Definition}

\newtheorem{example}[theorem]{Example}
\theoremstyle{remark}
\newtheorem{remark}[theorem]{Remark}

\newcommand{\R}{\mathbb{R}}
\newcommand{\E}{\mathbb{E}}
\newcommand{\Pp}{\mathbb{P}}
\newcommand{\bone}{\mathbf{1}}
\newcommand{\KL}{\mathrm{KL}}
\newcommand{\TV}{\mathrm{TV}}
\newcommand{\sign}{\mathrm{sign}}
\newcommand{\method}{\textsc{FairProj}}
\DeclareMathOperator*{\argmin}{arg\,min}

\icmltitlerunning{Projected Boosting with Fairness Constraints}

\begin{document}

\twocolumn[
  \icmltitle{Projected Boosting with Fairness Constraints:\\
  Quantifying the Cost of Fair Training Distributions}

  \icmlsetsymbol{equal}{*}

  \begin{icmlauthorlist}
    \icmlauthor{Amir Asiaee}{vumc}
    \icmlauthor{Kaveh Aryan}{kcl}
  \end{icmlauthorlist}

  \icmlaffiliation{vumc}{Department of Biostatistics, Vanderbilt University Medical Center, Nashville, TN 37232, USA}
  \icmlaffiliation{kcl}{Department of Informatics, King's College London, WC2R 2LS London, UK}

  \icmlcorrespondingauthor{Amir Asiaee}{amir.asiaeetaheri@vumc.org}

  \icmlkeywords{Fairness, Boosting, Algorithmic Fairness, Constrained Optimization, Machine Learning}

  \vskip 0.3in
]

\printAffiliationsAndNotice{}

\begin{abstract}
Boosting algorithms enjoy strong theoretical guarantees: when weak learners maintain positive edge, AdaBoost achieves geometric decrease of exponential loss. We study how to incorporate group fairness constraints into boosting while preserving analyzable training dynamics. Our approach, \method{}, projects the ensemble-induced exponential-weights distribution onto a convex set of distributions satisfying fairness constraints (as a \emph{reweighting surrogate}), then trains weak learners on this fair distribution. The key theoretical insight is that projecting the training distribution reduces the \emph{effective edge} of weak learners by a quantity controlled by the KL-divergence of the projection. We prove an exponential-loss bound where the convergence rate depends on weak learner edge minus a ``fairness cost'' term $\delta_t = \sqrt{\KL(w^t \| q^t)/2}$. This directly quantifies the accuracy-fairness tradeoff in boosting dynamics. Experiments on standard benchmarks validate the theoretical predictions and demonstrate competitive fairness-accuracy tradeoffs with stable training curves.
\end{abstract}

\section{Introduction}
\label{sec:intro}

AdaBoost and gradient boosting remain highly effective for tabular classification, offering interpretable models with well-understood training dynamics \citep{freund1997decision,schapire2013explaining}. A key advantage is the clean theoretical guarantee: under mild assumptions on weak learner quality, the exponential loss decreases geometrically at a rate determined by the weak learner edge.

Deploying machine learning in high-stakes domains requires models satisfying fairness constraints \citep{hardt2016equality}. A natural approach is to modify the sample weights used to train weak learners based on group-conditioned performance. However, such modifications break the mathematical machinery underlying AdaBoost's convergence proof: the weight update no longer minimizes the per-round exponential loss bound, and the resulting algorithms lack clean theoretical guarantees.

\paragraph{Our approach.}
We propose \method{}, which separates the ensemble-induced distribution from the fair training distribution. At each round $t$:
\begin{enumerate}
    \item We compute $q^t$, the exponential-weights distribution induced by the current ensemble $f_{t-1}$ (this is exactly the distribution AdaBoost would use).
    \item We \emph{project} $q^t$ onto a convex set of distributions satisfying fairness constraints, yielding fair training weights $w^t$.
    \item We train weak learner $h_t$ on $w^t$, but compute the boosting coefficient $\alpha_t$ using error measured under $q^t$.
\end{enumerate}

This design preserves the standard AdaBoost recursion for exponential loss (since $\alpha_t$ is computed under $q^t$), while allowing fairness constraints to influence weak learner selection. The \emph{cost} of fairness manifests as a reduction in effective edge: if $h_t$ achieves edge $\gamma^{(w)}_t$ under the fair distribution $w^t$, its edge under $q^t$ may be smaller due to distribution mismatch. We bound this gap using the KL-divergence of the projection.

\paragraph{Contributions.}
\begin{enumerate}
    \item \textbf{Projection-based fair boosting.} We introduce \method{}, which projects ensemble-induced distributions onto fairness constraint sets. We clarify that this enforces fairness as a \emph{reweighting surrogate}---constraints on training distributions, not direct guarantees on classifier fairness.

    \item \textbf{Edge transfer analysis.} We prove that the effective edge under $q^t$ is at least $\gamma^{(w)}_t - \delta_t$ where $\delta_t = \sqrt{\KL(w^t \| q^t)/2}$ (Lemma~\ref{lem:edge-transfer}). This directly quantifies the cost of fairness in terms of a divergence penalty.

    \item \textbf{Exponential loss bound.} We prove that \method{} achieves $L_{\exp}(f_T) \leq n \exp(-2 \sum_{t=1}^T (\gamma^{(w)}_t - \delta_t)^2)$ when $\gamma^{(w)}_t > \delta_t$ (Theorem~\ref{thm:main}). This recovers AdaBoost-like guarantees with an explicit fairness cost.

    \item \textbf{Experiments.} We evaluate on standard benchmarks, validating theoretical predictions and demonstrating competitive fairness-accuracy tradeoffs.
\end{enumerate}

\paragraph{What this paper does \emph{not} claim.}
We do not claim that projecting training distributions directly enforces fairness of the final classifier. The constraints are a \emph{surrogate} that influences weak learner selection. Whether this translates to classifier fairness is an empirical question we investigate experimentally.

\section{Related Work}
\label{sec:related}

\paragraph{Boosting as optimization.}
\citet{freund1997decision} introduced AdaBoost and proved that when weak learners achieve edge $\gamma_t > 0$, the training error decreases exponentially as $\exp(-2\sum_t \gamma_t^2)$. \citet{schapire2013explaining} later showed that AdaBoost can be viewed as coordinate descent on exponential loss, or equivalently as mirror descent with entropic regularization \citep{hazan2016introduction}. This optimization perspective motivates our approach: we add fairness constraints to the mirror descent formulation, with KL-projection being the natural choice in the entropic geometry.

\paragraph{Fairness definitions and approaches.}
\citet{dwork2012fairness} introduced individual fairness, requiring similar individuals to receive similar predictions. \citet{hardt2016equality} defined group fairness notions: \emph{equalized odds} requires equal true positive and false positive rates across groups, while \emph{equal opportunity} relaxes this to only true positive rates. \citet{zemel2013learning} proposed learning fair representations that encode task-relevant information while obfuscating group membership. \citet{hardt2016equality} also proposed post-processing predictions by group-specific thresholds to achieve equalized odds. Our work takes an in-processing approach, modifying training distributions rather than representations or predictions.

\paragraph{Reductions approach to fair classification.}
\citet{agarwal2018reductions} frame fair classification as a two-player game between a learner (minimizing loss) and an auditor (maximizing constraint violation), solved via exponentiated gradient on the Lagrangian. This reduction yields classifiers with provable $\epsilon$-approximate fairness by iteratively reweighting a cost-sensitive classification oracle. Our work differs in two ways: (1) we operate within the sequential boosting framework rather than the batch reduction setting, and (2) we analyze how fairness constraints affect the \emph{convergence rate} rather than equilibrium properties.

\paragraph{Fairness-aware boosting.}
\citet{kamiran2012preprocessing} proposed preprocessing approaches including reweighing training examples by group-label combinations to remove discrimination before training any classifier. \citet{song2024fab} introduced Fair AdaBoost (FAB), which modifies sample weights using a fairness-aware reweighting technique and derives an upper bound on a combined loss quantifying both error rate and unfairness. However, FAB modifies both the weight update \emph{and} the coefficient computation, which breaks the clean exponential-loss recursion underlying AdaBoost's convergence proof. \citet{cruz2022fairgbm} proposed FairGBM for gradient boosting, using smooth convex proxies for non-differentiable fairness metrics and a proxy-Lagrangian formulation with dual ascent. In contrast to these approaches, \method{} preserves the standard $\alpha_t$ calculation under the ensemble-induced distribution $q^t$, enabling our edge-transfer analysis that explicitly quantifies how distribution projection affects convergence.

\paragraph{Adaptive reweighting for fairness.}
Beyond boosting, adaptive reweighting methods with convergence guarantees have been developed. \citet{chai2022adaptive} proposed learning adaptive weights for each sample to achieve group-level balance, deriving a closed-form solution for weight assignment with theoretical convergence guarantees. \citet{hu2023adaptive} developed adaptive priority reweighing to improve fair classifier generalization. These methods operate outside the boosting framework; our contribution is showing how to incorporate reweighting into boosting while preserving its characteristic exponential-loss analysis and deriving explicit bounds on the ``fairness cost.''

\paragraph{Cost of fairness and impossibility results.}
\citet{corbett2017algorithmic} studied the cost of imposing fairness constraints on algorithmic decision-making, showing fundamental tradeoffs between accuracy and various fairness definitions. \citet{kleinberg2016inherent} proved that calibration and balance (equal false positive/negative rates across groups) cannot be simultaneously achieved except when base rates are equal. \citet{chouldechova2017fair} proved a related impossibility: predictive parity and equal error rates are incompatible unless base rates match. Our theoretical contribution makes the accuracy-fairness tradeoff explicit within boosting: the ``fairness cost'' $\delta_t = \sqrt{\KL(w^t \| q^t)/2}$ directly quantifies how much convergence rate is sacrificed for distributional fairness constraints.

\section{Preliminaries}
\label{sec:prelim}

\subsection{Problem Setup}

We consider binary classification with training data $\{(x_i, a_i, y_i)\}_{i=1}^n$ where $x_i \in \mathcal{X}$, $a_i \in \mathcal{A}$ is a protected attribute, and $y_i \in \{-1, +1\}$. An ensemble classifier is $f_T(x) = \sum_{t=1}^T \alpha_t h_t(x)$ with weak learners $h_t: \mathcal{X} \to \{-1, +1\}$ and coefficients $\alpha_t > 0$.

\subsection{AdaBoost and the Exponential-Weights Distribution}

The following definition is central to our analysis.

\begin{definition}[Exponential-weights distribution]
\label{def:q-dist}
Let $f_{t-1}(x) = \sum_{s=1}^{t-1} \alpha_s h_s(x)$ be the ensemble after $t-1$ rounds, with $f_0 \equiv 0$. The \emph{exponential-weights distribution} induced by $f_{t-1}$ is:
\begin{equation}
    q_i^t := \frac{\exp(-y_i f_{t-1}(x_i))}{\sum_{j=1}^n \exp(-y_j f_{t-1}(x_j))}.
    \label{eq:q-def}
\end{equation}
\end{definition}

The exponential loss satisfies a key recursion:
\begin{equation}
    L_{\exp}(f_t) = L_{\exp}(f_{t-1}) \cdot \E_{i \sim q^t}\left[\exp(-\alpha_t y_i h_t(x_i))\right].
    \label{eq:exp-recursion}
\end{equation}

Standard AdaBoost trains $h_t$ on $q^t$, sets $\alpha_t = \frac{1}{2}\ln\frac{1-\varepsilon_t}{\varepsilon_t}$ where $\varepsilon_t = \Pp_{i \sim q^t}(y_i \neq h_t(x_i))$, and achieves:
\begin{equation}
    \E_{i \sim q^t}\left[\exp(-\alpha_t y_i h_t(x_i))\right] = 2\sqrt{\varepsilon_t(1-\varepsilon_t)} \leq \exp(-2\gamma_t^2),
\end{equation}
where $\gamma_t = \frac{1}{2} - \varepsilon_t$ is the edge.

\subsection{Fairness Constraints as a Reweighting Surrogate}
\label{sec:fairness-constraints}

We consider constraints expressible as linear moment conditions on training distributions.

\begin{definition}[Fairness constraint set]
\label{def:constraint-set}
Let $g: [n] \to \R^K$ be bounded constraint features with $|g_k(i)| \leq G$ for all $k, i$. The \emph{fairness constraint set} with slack $\epsilon > 0$ is:
\begin{equation}
    \mathcal{C}_\epsilon = \left\{ w \in \Delta_n : \left| \sum_{i=1}^n w_i g_k(i) \right| \leq \epsilon, \ \forall k \in [K] \right\}.
    \label{eq:constraint-set}
\end{equation}
\end{definition}

\begin{remark}[Reweighting surrogate, not classifier fairness]
\label{rem:surrogate}
The constraint $|\langle w, g_k \rangle| \leq \epsilon$ bounds a weighted average over training examples, \emph{not} a property of the classifier's predictions. We call this a \emph{reweighting surrogate} for fairness. The intuition is that training on fair distributions may encourage weak learners that treat groups more equitably, but this is not a formal guarantee. We investigate the empirical relationship in Section~\ref{sec:experiments}.
\end{remark}

\begin{example}[Equal opportunity surrogate]
For binary $A \in \{0, 1\}$, define:
\begin{equation}
    g(i) = \bone[y_i = 1] \cdot \left(\bone[a_i = 1] - \bone[a_i = 0]\right),
\end{equation}
so that $\langle w, g \rangle = w(A{=}1,Y{=}1) - w(A{=}0,Y{=}1)$. The constraint $|\langle w, g \rangle| \leq \epsilon$ bounds the difference in \emph{total} weight placed on positive examples across groups.
\end{example}

\section{Method: Projected Fair Boosting}
\label{sec:method}

\subsection{Core Idea: Project, Train, but Compute $\alpha$ on $q$}

The key insight is to \emph{separate} the distribution used to train weak learners (fair) from the distribution used to compute boosting coefficients (ensemble-induced). This preserves the exponential-loss recursion while allowing fairness to influence weak learner selection.

At each round $t$:
\begin{enumerate}
    \item Compute $q^t$ from $f_{t-1}$ via \eqref{eq:q-def}.
    \item Project: $w^t = \argmin_{w \in \mathcal{C}_\epsilon} \KL(w \| q^t)$.
    \item Train: fit $h_t$ using sample weights $w^t$.
    \item Evaluate under $q^t$: compute $\varepsilon_t^{(q)} = \sum_i q_i^t \bone[y_i \neq h_t(x_i)]$.
    \item Set $\alpha_t = \frac{1}{2}\ln\frac{1-\varepsilon_t^{(q)}}{\varepsilon_t^{(q)}}$ and update $f_t = f_{t-1} + \alpha_t h_t$.
\end{enumerate}

Because $\alpha_t$ is computed using error under $q^t$, the recursion \eqref{eq:exp-recursion} remains valid. The fairness projection affects only which weak learner $h_t$ is selected, not the exponential-loss analysis.

\subsection{KL Projection: Closed-Form Solution}

The projection $w^t = \argmin_{w \in \mathcal{C}_\epsilon} \KL(w \| q^t)$ has a closed-form dual characterization.

\begin{lemma}[KL projection via dual]
\label{lem:kl-dual}
Let $q \in \Delta_n$ and $\mathcal{C}_\epsilon$ be the constraint set \eqref{eq:constraint-set}. Define:
\begin{equation}
    Z(\lambda) := \sum_{i=1}^n q_i \exp\left(-\sum_{k=1}^K \lambda_k g_k(i)\right).
\end{equation}
The optimal dual variables solve:
\begin{equation}
    \lambda^* \in \argmin_{\lambda \in \R^K} \left\{ \log Z(\lambda) + \epsilon \|\lambda\|_1 \right\}.
    \label{eq:dual-problem}
\end{equation}
The projected distribution is:
\begin{equation}
    w_i^* = \frac{q_i \exp(-(\lambda^*)^\top g(i))}{Z(\lambda^*)},
    \label{eq:w-from-dual}
\end{equation}
and the projection divergence is:
\begin{equation}
    \KL(w^* \| q) = -\log Z(\lambda^*) - \epsilon \|\lambda^*\|_1.
    \label{eq:kl-value}
\end{equation}
\end{lemma}

\begin{proof}
The Lagrangian for \eqref{eq:constraint-set} with KL objective is:
\begin{align}
    \mathcal{L}(w, \lambda^+, \lambda^-) &= \sum_i w_i \log\frac{w_i}{q_i} + \sum_k \lambda_k^+ \left(\langle w, g_k \rangle - \epsilon\right) \nonumber \\
    &\quad + \sum_k \lambda_k^- \left(-\langle w, g_k \rangle - \epsilon\right).
\end{align}
Stationarity in $w_i$ gives $w_i \propto q_i \exp(-(\lambda^+ - \lambda^-)^\top g(i))$. Setting $\lambda = \lambda^+ - \lambda^-$ and substituting back yields the dual. The value \eqref{eq:kl-value} follows from strong duality; note $\log Z(0) = 0$, so the optimal value is non-positive, hence $\KL(w^* \| q) \geq 0$ as required.
\end{proof}

\begin{remark}[Computational cost]
The dual problem \eqref{eq:dual-problem} is $K$-dimensional convex optimization. For common fairness notions, $K \in \{1, 2\}$. We solve it via L-BFGS-B or projected gradient descent, warm-starting with the previous round's solution for efficiency.
\end{remark}

\subsection{The \method{} Algorithm}

\begin{algorithm}[tb]
\caption{\method{}: Projected Fair Boosting}
\label{alg:fairproj}
\begin{algorithmic}[1]
\STATE {\bfseries Input:} Data $\{(x_i, a_i, y_i)\}_{i=1}^n$, weak learner class $\mathcal{H}$, rounds $T$, slack $\epsilon$
\STATE {\bfseries Initialize:} $f_0(x) \equiv 0$, $\lambda^0 \leftarrow \mathbf{0} \in \R^K$
\FOR{$t = 1$ to $T$}
    \STATE Compute $q^t$ from $f_{t-1}$: \quad $q_i^t \propto \exp(-y_i f_{t-1}(x_i))$
    \STATE {\bfseries Project:} Solve \eqref{eq:dual-problem} (warm-start with $\lambda^{t-1}$) to get $\lambda^t$
    \STATE Set $w_i^t \propto q_i^t \exp(-(\lambda^t)^\top g(i))$
    \STATE Compute $\delta_t \leftarrow \sqrt{\KL(w^t \| q^t) / 2}$ via \eqref{eq:kl-value}
    \STATE {\bfseries Train:} Fit $h_t \in \mathcal{H}$ on $\{(x_i, y_i)\}$ with weights $w^t$
    \STATE {\bfseries Evaluate on $q^t$:} $\varepsilon_t^{(q)} \leftarrow \sum_i q_i^t \bone[y_i \neq h_t(x_i)]$
    \IF{$\varepsilon_t^{(q)} \geq 0.5$}
        \STATE {\bfseries break} \COMMENT{No useful weak learner under $q^t$}
    \ENDIF
    \STATE Set $\alpha_t \leftarrow \frac{1}{2} \ln \frac{1 - \varepsilon_t^{(q)}}{\varepsilon_t^{(q)}}$
    \STATE Update $f_t \leftarrow f_{t-1} + \alpha_t h_t$
\ENDFOR
\STATE {\bfseries Output:} Classifier $F_T(x) = \sign(f_T(x))$
\end{algorithmic}
\end{algorithm}

\paragraph{Warm-starting as ``forward-aware'' computation.}
We warm-start the dual solver with $\lambda^{t-1}$ from the previous round. This is a computational optimization that reduces solver iterations. We do \emph{not} claim this provides additional theoretical benefits; the guarantee in Theorem~\ref{thm:main} holds regardless of how the projection is solved.

\section{Theoretical Analysis}
\label{sec:theory}

\subsection{Edge Transfer Under Distribution Mismatch}

The key technical result bounds how edge changes when evaluated under different distributions.

\begin{definition}[Edge]
For weak learner $h$ and distribution $p \in \Delta_n$, the \emph{edge} is:
\begin{equation}
    \gamma(p) := \frac{1}{2} \E_{i \sim p}[y_i h(x_i)] = \frac{1}{2} - \Pp_{i \sim p}(y_i \neq h(x_i)).
\end{equation}
\end{definition}

\begin{lemma}[Edge transfer bound]
\label{lem:edge-transfer}
For any $h: \mathcal{X} \to \{-1, +1\}$ and distributions $p, q \in \Delta_n$:
\begin{equation}
    |\gamma(p) - \gamma(q)| \leq \TV(p, q) \leq \sqrt{\frac{1}{2} \KL(p \| q)},
    \label{eq:edge-transfer}
\end{equation}
where $\TV(p, q) = \frac{1}{2}\sum_i |p_i - q_i|$ is total variation distance.
\end{lemma}

\begin{proof}
Let $m_i = y_i h(x_i) \in \{-1, +1\}$ denote margins. Then:
\begin{align}
    |\gamma(p) - \gamma(q)| &= \frac{1}{2} \left| \sum_i (p_i - q_i) m_i \right| \\
    &\leq \frac{1}{2} \sum_i |p_i - q_i| \cdot |m_i| = \TV(p, q).
\end{align}
The second inequality is Pinsker's inequality.
\end{proof}

\begin{corollary}
\label{cor:edge-transfer}
If $w^t$ is obtained by projecting $q^t$ onto $\mathcal{C}_\epsilon$, and $h_t$ achieves edge $\gamma_t^{(w)} := \gamma(w^t)$ under $w^t$, then its edge under $q^t$ satisfies:
\begin{equation}
    \gamma_t^{(q)} := \gamma(q^t) \geq \gamma_t^{(w)} - \delta_t, \quad \text{where } \delta_t = \sqrt{\frac{1}{2}\KL(w^t \| q^t)}.
\end{equation}
\end{corollary}

\subsection{Main Theorem: Exponential Loss with Fairness Cost}

\begin{theorem}[Exponential loss bound with fairness cost]
\label{thm:main}
Consider Algorithm~\ref{alg:fairproj} running for $T$ rounds. Let $\gamma_t^{(w)} = \gamma(w^t)$ be the edge of $h_t$ under the fair training distribution, and $\delta_t = \sqrt{\KL(w^t \| q^t)/2}$ be the fairness cost at round $t$. Suppose $\gamma_t^{(w)} > \delta_t$ for all $t$ (the weak learner has positive effective edge). Then:
\begin{equation}
    L_{\exp}(f_T) \leq n \exp\left(-2 \sum_{t=1}^T (\gamma_t^{(w)} - \delta_t)^2\right).
    \label{eq:main-bound}
\end{equation}
Consequently, the training error satisfies:
\begin{equation}
    \frac{1}{n} \sum_{i=1}^n \bone[\sign(f_T(x_i)) \neq y_i] \leq \exp\left(-2 \sum_{t=1}^T (\gamma_t^{(w)} - \delta_t)^2\right).
\end{equation}
Moreover, by construction, $|\langle w^t, g_k \rangle| \leq \epsilon$ for all $k, t$.
\end{theorem}

\begin{proof}
By the exponential-loss recursion (Definition~\ref{def:q-dist}):
\begin{equation}
    L_{\exp}(f_t) = L_{\exp}(f_{t-1}) \cdot \E_{i \sim q^t}[\exp(-\alpha_t y_i h_t(x_i))].
\end{equation}
With $\alpha_t = \frac{1}{2}\ln\frac{1-\varepsilon_t^{(q)}}{\varepsilon_t^{(q)}}$ and $\varepsilon_t^{(q)} = \frac{1}{2} - \gamma_t^{(q)}$, the standard AdaBoost calculation gives:
\begin{equation}
    \E_{i \sim q^t}[\exp(-\alpha_t y_i h_t(x_i))] = 2\sqrt{\varepsilon_t^{(q)}(1-\varepsilon_t^{(q)})} = \sqrt{1 - 4(\gamma_t^{(q)})^2}.
\end{equation}
Using $\sqrt{1-x} \leq e^{-x/2}$ for $x \in [0, 1]$:
\begin{equation}
    \sqrt{1 - 4(\gamma_t^{(q)})^2} \leq \exp(-2(\gamma_t^{(q)})^2).
\end{equation}
By Corollary~\ref{cor:edge-transfer}, $\gamma_t^{(q)} \geq \gamma_t^{(w)} - \delta_t$. Since $x \mapsto x^2$ is increasing for $x > 0$ and we assume $\gamma_t^{(w)} > \delta_t$:
\begin{equation}
    \exp(-2(\gamma_t^{(q)})^2) \leq \exp(-2(\gamma_t^{(w)} - \delta_t)^2).
\end{equation}
Telescoping from $L_{\exp}(f_0) = n$:
\begin{align}
    L_{\exp}(f_T) &\leq n \prod_{t=1}^T \exp(-2(\gamma_t^{(w)} - \delta_t)^2) \nonumber \\
    &= n \exp\left(-2\sum_{t=1}^T (\gamma_t^{(w)} - \delta_t)^2\right).
\end{align}
The training error bound follows from $\bone[\sign(f) \neq y] \leq \exp(-yf)$. The constraint satisfaction $|\langle w^t, g_k \rangle| \leq \epsilon$ holds by definition of $\mathcal{C}_\epsilon$.
\end{proof}

\paragraph{Interpretation.}
Theorem~\ref{thm:main} shows that \method{} achieves AdaBoost-like exponential convergence, but with an \emph{effective edge} $\gamma_t^{(w)} - \delta_t$ instead of $\gamma_t^{(w)}$. The term $\delta_t = \sqrt{\KL(w^t \| q^t)/2}$ is the ``fairness cost''---the price paid for training on a fair distribution rather than the ensemble-induced distribution.

\begin{remark}[Comparison to unconstrained AdaBoost]
When $\epsilon \to \infty$, the constraint set $\mathcal{C}_\epsilon = \Delta_n$, so $w^t = q^t$ and $\delta_t = 0$. The bound reduces to standard AdaBoost: $L_{\exp}(f_T) \leq n \exp(-2\sum_t (\gamma_t)^2)$.
\end{remark}

\begin{remark}[Tradeoff controlled by $\epsilon$]
Smaller $\epsilon$ enforces tighter surrogate fairness but may increase $\delta_t$ (larger divergence from $q^t$), reducing the effective edge. Larger $\epsilon$ allows $w^t \approx q^t$ but provides weaker fairness enforcement. This tradeoff is explicit in the bound.
\end{remark}

\subsection{Sufficient Condition for Positive Effective Edge}

The assumption $\gamma_t^{(w)} > \delta_t$ requires that weak learners trained on fair distributions have edge exceeding the fairness cost. We provide a sufficient condition.

\begin{proposition}
\label{prop:sufficient}
Suppose the weak learner class $\mathcal{H}$ achieves edge at least $\gamma_{\min} > 0$ under any distribution in $\mathcal{C}_\epsilon$. Suppose further that $\KL(w^t \| q^t) \leq D$ for some $D > 0$ (e.g., by choosing $\epsilon$ not too small). Then:
\begin{equation}
    \gamma_t^{(w)} - \delta_t \geq \gamma_{\min} - \sqrt{D/2}.
\end{equation}
If $\gamma_{\min} > \sqrt{D/2}$, the effective edge is positive.
\end{proposition}

This formalizes the intuition that if weak learners are ``strong enough'' and fairness constraints are ``not too tight,'' convergence is guaranteed.

\section{Experiments}
\label{sec:experiments}

We evaluate \method{} on standard fairness benchmarks, focusing on three questions:
\begin{enumerate}
    \item Does Theorem~\ref{thm:main} accurately predict training dynamics?
    \item Does the reweighting surrogate improve classifier fairness?
    \item How do fairness-accuracy tradeoffs compare to baselines?
\end{enumerate}

\subsection{Experimental Setup}

\paragraph{Datasets.}
We evaluate on three standard tabular benchmarks: Adult (income prediction; 48,842 examples), German Credit (credit risk; 1,000 examples), and COMPAS (two-year recidivism; 5,278 examples after standard filtering). We use sex as the protected attribute for Adult and German, and race (African-American vs Caucasian) for COMPAS.

\paragraph{Metrics.}
Accuracy for performance; equal opportunity gap (EOpp $= |\text{TPR}_{a=1} - \text{TPR}_{a=0}|$) and demographic parity gap (DP $= |\Pp(\hat{Y}{=}1 \mid A{=}1) - \Pp(\hat{Y}{=}1 \mid A{=}0)|$) for fairness.

\paragraph{Fairness surrogate.}
For \method{}, our primary surrogate is a \emph{mass-balanced} equal opportunity constraint that encourages $w(A{=}1, Y{=}1) \approx w(A{=}0, Y{=}1)$ (Appendix~\ref{app:constraints}). We also report results for the analogous mass-balanced demographic parity surrogate. Intuitively, these reweightings shift weak-learner training toward under-represented group--label mass, which is aligned with reducing the corresponding gap.

\paragraph{Baselines.}
\begin{itemize}
    \item \textbf{AdaBoost}: Standard AdaBoost with decision stumps.
    \item \textbf{Reweighing}: Preprocessing-based group reweighting \citep{kamiran2012preprocessing}.
    \item \textbf{Exponentiated Gradient (EG)}: Fairlearn reductions with either equal opportunity or demographic parity constraints \citep{agarwal2018reductions}.
\end{itemize}

\paragraph{Protocol.}
100 boosting rounds, decision stumps (depth-1 trees) as weak learners, and 80/20 train/test splits. We report mean $\pm$ std over 10 random splits (seeds 42--51).

\subsection{Main Results}

Table~\ref{tab:main} reports accuracy and EOpp gap across datasets using the mass-balanced EOpp surrogate, and Figure~\ref{fig:pareto} shows the corresponding Pareto frontiers. Figure~\ref{fig:pareto-dp} shows analogous tradeoffs for demographic parity using the mass-balanced DP surrogate.

\begin{table*}[t]
\caption{Accuracy and Equal Opportunity gap (EOpp) across datasets. \method{} uses the mass-balanced EOpp surrogate. We report mean $\pm$ std over 10 random train/test splits (seeds 42--51). \textbf{Rounds} is the average number of boosting rounds before termination (AdaBoost/Reweighing always use 100 rounds; EG is not iterative boosting and is marked ``--'').}
\label{tab:main}
\centering
\begin{small}
\begin{sc}
\begin{tabular}{llccc}
\toprule
Dataset & Method & Accuracy & EOpp Gap & Rounds \\
\midrule
Adult (A=sex) & AdaBoost & $0.854 \pm 0.003$ & $0.185 \pm 0.033$ & 100.0 \\
Adult (A=sex) & Reweighing & $0.853 \pm 0.004$ & $0.115 \pm 0.034$ & 100.0 \\
Adult (A=sex) & \method{} ($\epsilon=0.25$) & $0.811 \pm 0.002$ & $0.043 \pm 0.030$ & 10.0 \\
Adult (A=sex) & \method{} ($\epsilon=0.10$) & $0.804 \pm 0.002$ & $0.029 \pm 0.014$ & 4.8 \\
Adult (A=sex) & EG ($\epsilon=0.05$) & $0.761 \pm 0.000$ & $0.000 \pm 0.000$ & -- \\
\midrule
German (A=sex) & AdaBoost & $0.745 \pm 0.026$ & $0.097 \pm 0.049$ & 100.0 \\
German (A=sex) & Reweighing & $0.749 \pm 0.025$ & $0.057 \pm 0.031$ & 100.0 \\
German (A=sex) & \method{} ($\epsilon=0.25$) & $0.733 \pm 0.018$ & $0.046 \pm 0.042$ & 25.7 \\
German (A=sex) & \method{} ($\epsilon=0.10$) & $0.735 \pm 0.016$ & $0.063 \pm 0.047$ & 17.4 \\
German (A=sex) & EG ($\epsilon=0.05$) & $0.700 \pm 0.000$ & $0.000 \pm 0.000$ & -- \\
\midrule
COMPAS (A=race) & AdaBoost & $0.669 \pm 0.009$ & $0.188 \pm 0.044$ & 100.0 \\
COMPAS (A=race) & Reweighing & $0.662 \pm 0.010$ & $0.179 \pm 0.047$ & 100.0 \\
COMPAS (A=race) & \method{} ($\epsilon=0.25$) & $0.669 \pm 0.009$ & $0.188 \pm 0.044$ & 100.0 \\
COMPAS (A=race) & \method{} ($\epsilon=0.10$) & $0.660 \pm 0.011$ & $0.169 \pm 0.046$ & 17.3 \\
COMPAS (A=race) & EG ($\epsilon=0.05$) & $0.586 \pm 0.017$ & $0.050 \pm 0.038$ & -- \\
\bottomrule
\end{tabular}
\end{sc}
\end{small}
\end{table*}

\begin{figure*}[t]
\centering
\includegraphics[width=\textwidth]{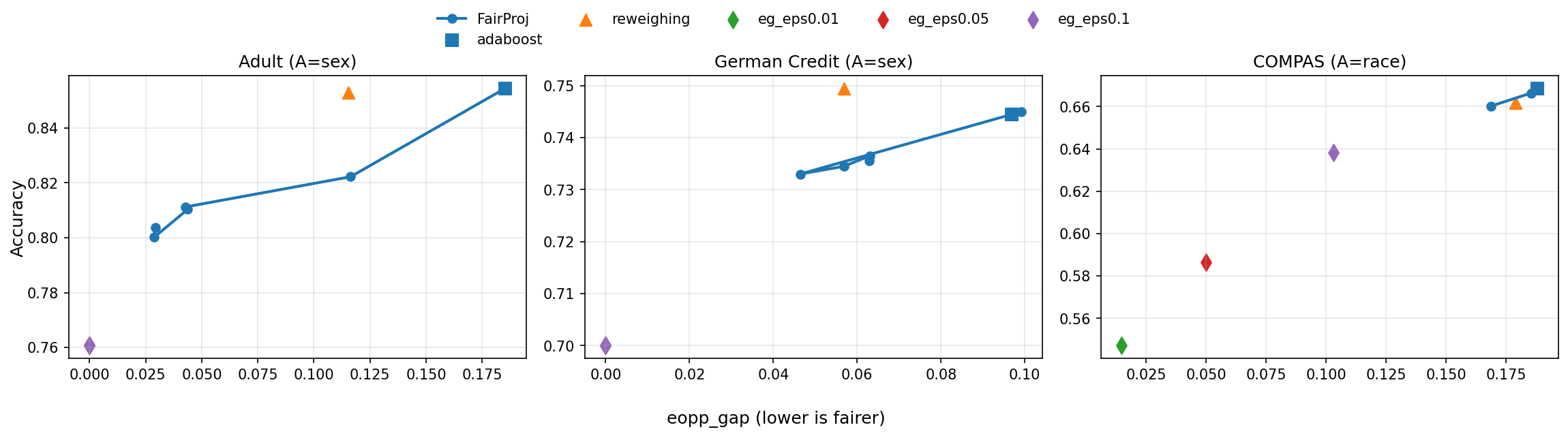}
\caption{EOpp fairness--accuracy tradeoffs across Adult (A=sex), German (A=sex), and COMPAS (A=race). \method{} traces a Pareto curve as $\epsilon$ varies (smaller is tighter). On Adult and German, tightening $\epsilon$ yields large EOpp reductions with moderate accuracy loss; on COMPAS the gains are smaller. EG achieves lower gaps but can incur a larger accuracy cost.}
\label{fig:pareto}
\end{figure*}

\begin{figure*}[t]
\centering
\includegraphics[width=\textwidth]{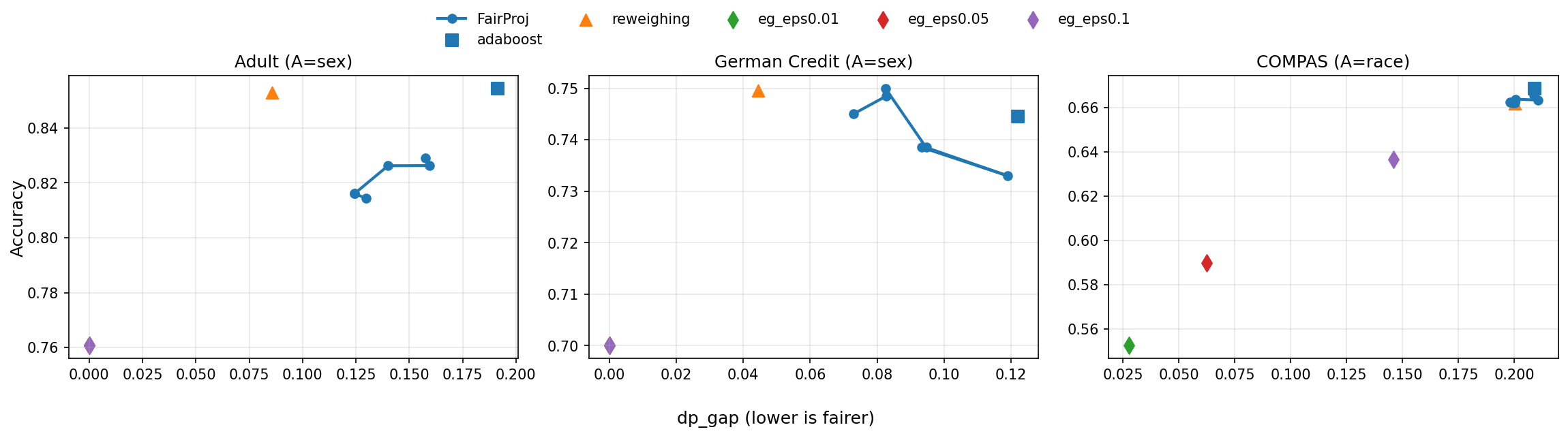}
\caption{DP fairness--accuracy tradeoffs across datasets for the mass-balanced DP surrogate. Smaller $\epsilon$ can reduce DP gap relative to unconstrained AdaBoost, at the cost of early termination and reduced accuracy. Reweighing is often competitive on DP since it directly targets group--label mass balancing.}
\label{fig:pareto-dp}
\end{figure*}

\paragraph{Demographic parity.}
Figure~\ref{fig:pareto-dp} shows that the same projection mechanism can be used with a DP surrogate. On Adult and German, tighter $\epsilon$ can reduce DP gap relative to AdaBoost, but the strongest reductions require smaller slacks than in the EOpp setting. Reweighing is often competitive on DP since it directly targets group--label mass, while EG can achieve lower DP gaps at a larger accuracy cost.

\paragraph{Across-dataset dynamics.}
Table~\ref{tab:main} also highlights the dataset-dependent stopping behavior induced by projection. For tight $\epsilon$, \method{} frequently terminates early because weak learners trained on $w^t$ lose positive edge under $q^t$; for looser $\epsilon$, the projection may be inactive and \method{} recovers unconstrained boosting (e.g., 100 rounds on COMPAS at $\epsilon=0.25$).

\paragraph{Key observations.}
(1) On Adult and German, \method{} achieves substantially lower EOpp gap than both AdaBoost and Reweighing, with a moderate accuracy reduction. On COMPAS, the improvement is smaller.
(2) Tightening $\epsilon$ increases the projection divergence (fairness cost) and triggers earlier stopping because weak learners trained on $w^t$ eventually lose positive edge under $q^t$ (i.e., $\varepsilon_t^{(q)} \to 0.5$).
(3) EG can achieve smaller EOpp gaps, but may incur a larger accuracy cost than \method{} at comparable gaps.
(4) \method{} also yields nontrivial DP tradeoffs (Figure~\ref{fig:pareto-dp}), illustrating that the projection framework extends beyond equal opportunity.

\subsection{Validating Theoretical Predictions}

We validate the edge-transfer analysis by examining training dynamics for \method{} with $\epsilon = 0.25$ on Adult (10 rounds before termination). Figure~\ref{fig:training-curves} visualizes these dynamics.

\begin{figure}[t]
\centering
\includegraphics[width=\columnwidth]{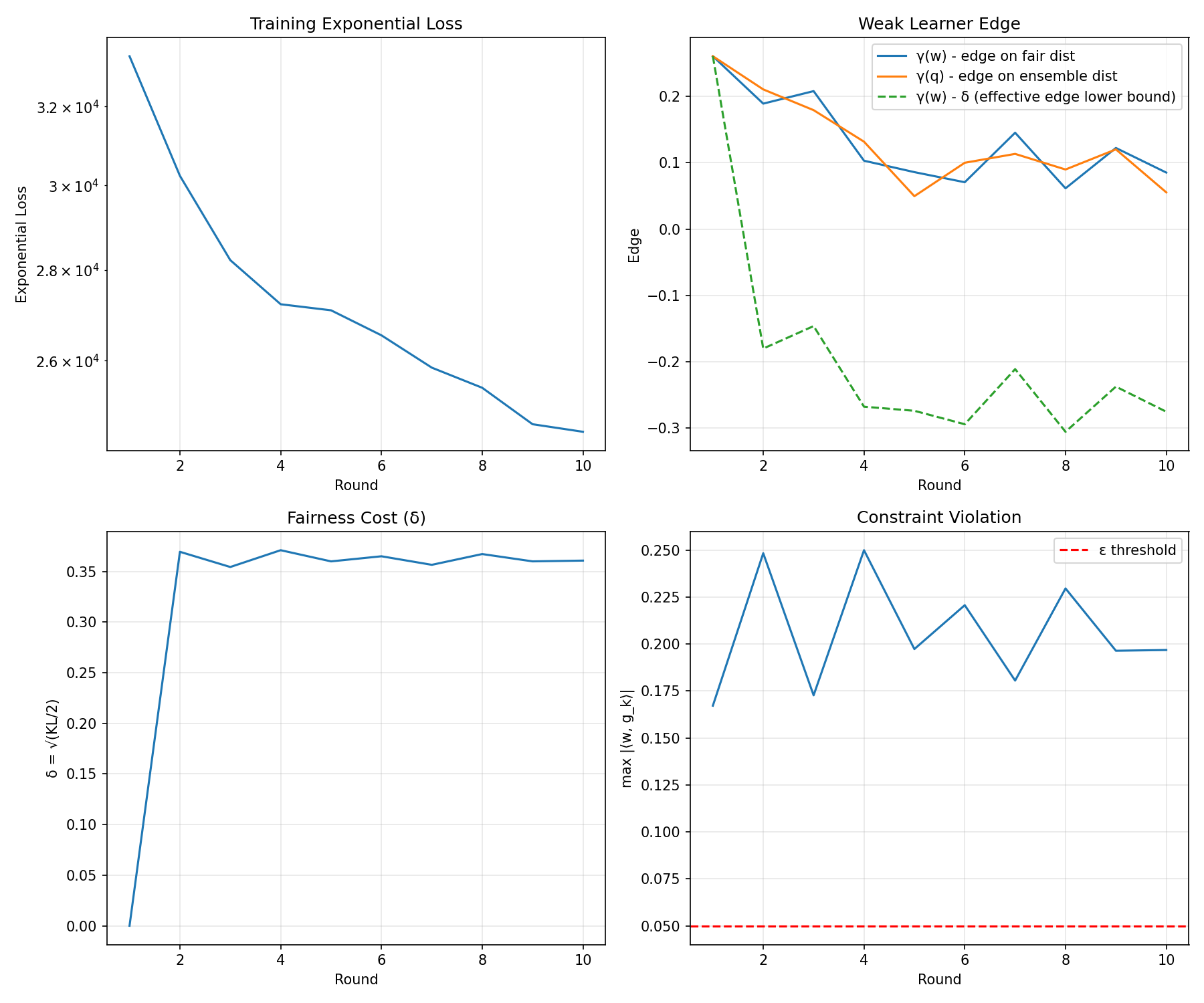}
\caption{Training dynamics for \method{} with $\epsilon = 0.25$ on Adult. \textbf{Top-left:} Exponential loss decreases monotonically. \textbf{Top-right:} Edge under fair distribution $\gamma_t^{(w)}$ (blue), edge under ensemble distribution $\gamma_t^{(q)}$ (orange), and effective edge lower bound $\gamma_t^{(w)} - \delta_t$ (green dashed). The bound from Lemma~\ref{lem:edge-transfer} holds at every round. \textbf{Bottom-left:} Fairness cost $\delta_t$ stabilizes after initial rounds. \textbf{Bottom-right:} Constraint violation stays below threshold $\epsilon$.}
\label{fig:training-curves}
\end{figure}

\paragraph{Edge transfer bound.}
We compute $\gamma_t^{(w)}$ (edge under fair distribution $w^t$), $\gamma_t^{(q)}$ (edge under ensemble distribution $q^t$), and $\delta_t = \sqrt{\KL(w^t \| q^t)/2}$. As shown in Figure~\ref{fig:training-curves} (top-right), the bound $\gamma_t^{(q)} \geq \gamma_t^{(w)} - \delta_t$ from Lemma~\ref{lem:edge-transfer} holds for all 10 rounds. The averages are $\gamma_t^{(w)} = 0.133$, $\gamma_t^{(q)} = 0.131$, and $\delta_t = 0.326$.

\paragraph{Monotonic loss decrease.}
Figure~\ref{fig:training-curves} (top-left) confirms that the exponential loss $L_{\exp}(f_t)$ decreases monotonically until the stopping criterion $\varepsilon_t^{(q)} \geq 0.5$ is met, as predicted by Theorem~\ref{thm:main}.

\paragraph{Constraint satisfaction.}
Figure~\ref{fig:training-curves} (bottom-right) shows the surrogate constraint $|\langle w^t, g_k \rangle|$ staying near $\epsilon$ when the projection is active, confirming feasibility.

\subsection{Why Does Mass-Balanced Reweighting Reduce EOpp Gap?}

The mass-balanced surrogate forces the training distribution to allocate comparable \emph{total} weight to positive examples from each group. This shifts weak-learner training toward the disadvantaged group's positives and empirically reduces EOpp gap. Still, \method{} does not directly constrain predictions, so the mapping from surrogate satisfaction to classifier fairness is empirical rather than guaranteed.
\begin{enumerate}
    \item The weak learner $h_t$ is trained to minimize \emph{weighted error}, not to directly equalize group-specific performance.
    \item The ensemble combines weak learners with coefficients $\alpha_t$ computed under $q^t$, not $w^t$.
    \item Classifier fairness depends on final predictions, which are a nonlinear function of the entire boosting trajectory.
\end{enumerate}

Methods like Exponentiated Gradient directly constrain the classifier's predictions via a game-theoretic reduction, achieving true classifier fairness at the cost of accuracy. Our contribution complements these approaches by quantifying how fairness-motivated reweighting impacts boosting dynamics through an explicit fairness cost term.

\subsection{Effect of $\epsilon$ on Training Dynamics}

\begin{table}[t]
\caption{Effect of $\epsilon$ on training dynamics on Adult. Smaller $\epsilon$ increases fairness cost $\delta_t$, reducing effective edge and causing earlier termination.}
\label{tab:epsilon}
\begin{center}
\begin{small}
\begin{sc}
\begin{tabular}{lccc}
\toprule
$\epsilon$ & Avg $\delta_t$ & Rounds & Accuracy \\
\midrule
0.40 & $0.000 \pm 0.000$ & 100.0 & $0.854 \pm 0.003$ \\
0.25 & $0.325 \pm 0.003$ & 10.0 & $0.811 \pm 0.002$ \\
0.15 & $0.477 \pm 0.003$ & 7.4 & $0.800 \pm 0.002$ \\
\bottomrule
\end{tabular}
\end{sc}
\end{small}
\end{center}
\end{table}

Table~\ref{tab:epsilon} shows that as $\epsilon$ decreases, the average fairness cost $\delta_t$ increases and the number of boosting rounds decreases (because weak learners trained on $w^t$ become less effective under $q^t$). This directly illustrates the tradeoff quantified by the edge-transfer analysis.

\section{Discussion and Limitations}
\label{sec:discussion}

\paragraph{Reweighting surrogate is not classifier fairness.}
We emphasize that \method{} enforces constraints on \emph{training distributions}, not the final classifier. The empirical question is whether this surrogate leads to fairer classifiers. Our experiments investigate this, but we do not provide theoretical guarantees on classifier fairness.

\paragraph{Effective edge assumption.}
Theorem~\ref{thm:main} requires $\gamma_t^{(w)} > \delta_t$. If fairness constraints are too tight ($\epsilon$ too small), the projection divergence $\delta_t$ may exceed the weak learner edge, violating this assumption. In practice, we terminate boosting when $\varepsilon_t^{(q)} \geq 0.5$.

\paragraph{Incompatible fairness definitions.}
Some fairness definitions are mutually incompatible \citep{kleinberg2016inherent,chouldechova2017fair}. Our framework handles this via slack $\epsilon$, but cannot satisfy impossible constraints.

\paragraph{Scalability.}
The $K$-dimensional dual solve adds $O(nK)$ cost per round. For typical $K \in \{1, 2\}$, this is negligible. For intersectional fairness with many groups, the cost grows.

\paragraph{Dataset-dependent effectiveness.}
The weaker results on COMPAS compared to Adult and German likely reflect dataset-specific factors. COMPAS has more balanced group sizes and base rates, so the mass-balanced surrogate induces smaller distributional shifts (the projection is often inactive at $\epsilon = 0.25$). Additionally, the features predictive of recidivism may be more entangled with race than income features are with sex, limiting what reweighting alone can achieve. This highlights that \method{}'s effectiveness depends on the alignment between the chosen surrogate and the dataset's fairness structure.

\section{Conclusion}
\label{sec:conclusion}

We introduced \method{}, a projection-based approach to fair boosting that cleanly separates the ensemble-induced distribution (for $\alpha_t$ computation) from the fair training distribution (for weak learner fitting). Our main theoretical contribution is the \emph{edge transfer} analysis (Lemma~\ref{lem:edge-transfer}) showing that projecting distributions reduces effective edge by at most $\sqrt{\KL(w^t \| q^t)/2}$. This yields an AdaBoost-like exponential-loss bound (Theorem~\ref{thm:main}) with an explicit ``fairness cost'' term.

We are explicit about what this paper does and does not claim: we enforce fairness as a reweighting surrogate on training distributions, not a direct guarantee on classifier fairness. Whether the surrogate translates to fairer classifiers is an empirical question we investigate. Our contribution is the theoretical framework for analyzing this tradeoff within boosting's convergence machinery.

\section*{Broader Impact}

This paper presents work whose goal is to advance the field of algorithmic fairness within machine learning. We develop methods for training classifiers under fairness constraints. Group fairness is one approach among many, and practitioners should carefully consider which fairness criteria are appropriate for their context. Our method provides a reweighting surrogate that may improve classifier fairness, but does not provide formal guarantees on final classifier behavior. We encourage practitioners to validate fairness properties on their specific deployment contexts.

\bibliographystyle{icml2026}
\bibliography{forward_aware_boosting_refs}

\newpage
\appendix
\onecolumn

\section{Constraint Features for Common Fairness Notions}
\label{app:constraints}

Let $n_a = |\{i : a_i = a\}|$, $n_a^+ = |\{i : a_i = a, y_i = 1\}|$, $n_a^- = |\{i : a_i = a, y_i = -1\}|$.

\paragraph{Demographic Parity Surrogate.}
\begin{equation}
    g_{\text{DP}}(i) = \bone[a_i = 1] - \bone[a_i = 0].
\end{equation}
This balances total weight across groups regardless of label.

\paragraph{Equal Opportunity Surrogate.}
\begin{equation}
    g_{\text{EOpp}}(i) = \bone[y_i = 1] \cdot \left(\bone[a_i = 1] - \bone[a_i = 0]\right).
\end{equation}
This balances total weight on positive examples across groups.

\paragraph{Equalized Odds Surrogate.}
Two constraints ($K=2$):
\begin{align}
    g_1(i) &= \bone[y_i = 1] \cdot \left(\bone[a_i = 1] - \bone[a_i = 0]\right), \\
    g_2(i) &= \bone[y_i = -1] \cdot \left(\bone[a_i = 1] - \bone[a_i = 0]\right).
\end{align}

\section{Proof Details}
\label{app:proofs}

\subsection{Proof of Lemma~\ref{lem:kl-dual}}

The primal problem is:
\begin{align}
    \min_{w \in \Delta_n} &\quad \KL(w \| q) \\
    \text{s.t.} &\quad |\langle w, g_k \rangle| \leq \epsilon, \quad k = 1, \ldots, K.
\end{align}

Introducing multipliers $\lambda_k^+, \lambda_k^- \geq 0$ for upper/lower bounds:
\begin{equation}
    \mathcal{L}(w, \lambda^+, \lambda^-) = \sum_i w_i \log\frac{w_i}{q_i} + \sum_k \lambda_k^+ (\langle w, g_k \rangle - \epsilon) + \sum_k \lambda_k^- (-\langle w, g_k \rangle - \epsilon).
\end{equation}

Setting $\frac{\partial \mathcal{L}}{\partial w_i} = 0$ (with normalization constraint via an additional multiplier $\nu$):
\begin{equation}
    \log\frac{w_i}{q_i} + 1 + (\lambda^+ - \lambda^-)^\top g(i) + \nu = 0.
\end{equation}
Thus $w_i \propto q_i \exp(-(\lambda^+ - \lambda^-)^\top g(i))$.

Setting $\lambda = \lambda^+ - \lambda^-$ and normalizing:
\begin{equation}
    w_i = \frac{q_i \exp(-\lambda^\top g(i))}{Z(\lambda)}, \quad Z(\lambda) = \sum_j q_j \exp(-\lambda^\top g(j)).
\end{equation}

The dual function is:
\begin{equation}
    d(\lambda^+, \lambda^-) = -\log Z(\lambda) - \epsilon(\|\lambda^+\|_1 + \|\lambda^-\|_1).
\end{equation}

Optimizing over $\lambda^+, \lambda^-$ with $\lambda = \lambda^+ - \lambda^-$ and using $\|\lambda^+\|_1 + \|\lambda^-\|_1 \geq \|\lambda\|_1$ (with equality when the signs are chosen correctly), we get:
\begin{equation}
    \max_\lambda d(\lambda) = \max_\lambda \{-\log Z(\lambda) - \epsilon \|\lambda\|_1\} = -\min_\lambda \{\log Z(\lambda) + \epsilon \|\lambda\|_1\}.
\end{equation}

By strong duality, $\KL(w^* \| q) = -d(\lambda^*) = -(-\log Z(\lambda^*) - \epsilon \|\lambda^*\|_1) = -\log Z(\lambda^*) - \epsilon \|\lambda^*\|_1$.

\textbf{Sign check:} At $\lambda = 0$, $Z(0) = 1$, so the dual objective is $\log(1) + 0 = 0$. The optimal value is $\leq 0$, hence $\log Z(\lambda^*) + \epsilon\|\lambda^*\|_1 \leq 0$, so $-\log Z(\lambda^*) - \epsilon\|\lambda^*\|_1 \geq 0$. This confirms $\KL(w^* \| q) \geq 0$.

\subsection{Proof of Theorem~\ref{thm:main}}

We prove each step carefully.

\textbf{Step 1: Exponential loss recursion.}
By definition of $q^t$:
\begin{equation}
    q_i^t = \frac{\exp(-y_i f_{t-1}(x_i))}{L_{\exp}(f_{t-1})}.
\end{equation}
Thus:
\begin{align}
    L_{\exp}(f_t) &= \sum_i \exp(-y_i f_t(x_i)) = \sum_i \exp(-y_i f_{t-1}(x_i)) \exp(-\alpha_t y_i h_t(x_i)) \\
    &= L_{\exp}(f_{t-1}) \sum_i q_i^t \exp(-\alpha_t y_i h_t(x_i)).
\end{align}

\textbf{Step 2: Per-round bound.}
With $\alpha_t = \frac{1}{2}\log\frac{1-\varepsilon_t^{(q)}}{\varepsilon_t^{(q)}}$ where $\varepsilon_t^{(q)} = \sum_i q_i^t \bone[y_i \neq h_t(x_i)]$:
\begin{align}
    \sum_i q_i^t \exp(-\alpha_t y_i h_t(x_i)) &= \varepsilon_t^{(q)} e^{\alpha_t} + (1-\varepsilon_t^{(q)}) e^{-\alpha_t} \\
    &= \varepsilon_t^{(q)} \sqrt{\frac{1-\varepsilon_t^{(q)}}{\varepsilon_t^{(q)}}} + (1-\varepsilon_t^{(q)}) \sqrt{\frac{\varepsilon_t^{(q)}}{1-\varepsilon_t^{(q)}}} \\
    &= 2\sqrt{\varepsilon_t^{(q)}(1-\varepsilon_t^{(q)})}.
\end{align}

With $\gamma_t^{(q)} = \frac{1}{2} - \varepsilon_t^{(q)}$, we have $\varepsilon_t^{(q)}(1-\varepsilon_t^{(q)}) = \frac{1}{4} - (\gamma_t^{(q)})^2$, so:
\begin{equation}
    2\sqrt{\varepsilon_t^{(q)}(1-\varepsilon_t^{(q)})} = \sqrt{1 - 4(\gamma_t^{(q)})^2} \leq \exp(-2(\gamma_t^{(q)})^2),
\end{equation}
using $\sqrt{1-x} \leq e^{-x/2}$.

\textbf{Step 3: Edge transfer.}
By Lemma~\ref{lem:edge-transfer}:
\begin{equation}
    \gamma_t^{(q)} \geq \gamma_t^{(w)} - \sqrt{\frac{1}{2}\KL(w^t \| q^t)} = \gamma_t^{(w)} - \delta_t.
\end{equation}

Since $(\gamma_t^{(q)})^2 \geq (\gamma_t^{(w)} - \delta_t)^2$ when $\gamma_t^{(w)} > \delta_t > 0$:
\begin{equation}
    \exp(-2(\gamma_t^{(q)})^2) \leq \exp(-2(\gamma_t^{(w)} - \delta_t)^2).
\end{equation}

\textbf{Step 4: Telescope.}
\begin{equation}
    L_{\exp}(f_T) = L_{\exp}(f_0) \prod_{t=1}^T \left(\sum_i q_i^t \exp(-\alpha_t y_i h_t(x_i))\right) \leq n \prod_{t=1}^T \exp(-2(\gamma_t^{(w)} - \delta_t)^2).
\end{equation}

\section{Implementation Details}
\label{app:implementation}

\paragraph{Dual solver.}
We solve the dual problem \eqref{eq:dual-problem} using scipy's L-BFGS-B optimizer. For the $\ell_1$ term, we use the smooth approximation $|\lambda_k| \approx \sqrt{\lambda_k^2 + \mu}$ with $\mu = 10^{-8}$, or the split-variable formulation with $\lambda = \lambda^+ - \lambda^-$ and non-negativity constraints.

\paragraph{Numerical stability.}
When computing $q^t \propto \exp(-y_i f_{t-1}(x_i))$, we subtract $\max_i(-y_i f_{t-1}(x_i))$ before exponentiating to avoid overflow.

\paragraph{Stopping criterion.}
We stop boosting when $\varepsilon_t^{(q)} \geq 0.5$ (no useful weak learner under $q^t$) or after $T$ rounds.

\end{document}